\documentclass{article}

\usepackage{amsthm,amsmath,amssymb,amsfonts,amsxtra,epsfig,verbatim,graphicx,float}
\usepackage{fullpage}

\RequirePackage[OT1]{fontenc}
\RequirePackage{amsthm,amsmath,natbib}
\RequirePackage[colorlinks,citecolor=blue,urlcolor=blue]{hyperref}
\RequirePackage{hypernat}

\numberwithin{equation}{section}
\theoremstyle{plain}
\newtheorem*{thm}{Theorem}

\theoremstyle{plain}
\newtheorem{theorem}{Theorem}[section]

\newtheorem{lemma}[theorem]{Lemma}

\theoremstyle{definition}

\theoremstyle{remark}

\newcommand{\argmin}{\operatornamewithlimits{argmin}}

\newcommand{\Var}{\operatornamewithlimits{Var}}

\newcommand{\card}{\operatornamewithlimits{card}}

\newcommand{\T}[1]{\ensuremath{{#1}^{\mbox{\sf\tiny T}}}}

\begin{document}

\title{Discovering Graphical Granger Causality \\ Using the Truncating Lasso Penalty}
\author{Ali Shojaie and George Michailidis \\ Department of Statistics, University of Michigan}
\date{}

\maketitle

\begin{abstract}
    Components of biological systems interact with each other in order to carry out vital cell functions. Such information can be used to improve estimation and inference, and to obtain better insights into the underlying cellular mechanisms. Discovering regulatory interactions among genes is therefore an important problem in systems biology. Whole-genome expression data over time provides an opportunity to determine how the expression levels of genes are affected by changes in transcription levels of other genes, and can therefore be used to discover regulatory interactions among genes.
    In this paper, we propose a novel penalization method, called \emph{truncating lasso}, for estimation of causal relationships from time-course gene expression data. The proposed penalty can correctly determine the order of the underlying time series, and improves the performance of the lasso-type estimators. Moreover, the resulting estimate provides information on the time lag between activation of transcription factors and their effects on regulated genes. We provide an efficient algorithm for estimation of model parameters, and show that the proposed method can consistently discover causal relationships in the large $p$, small $n$ setting. The performance of the proposed model is evaluated favorably in simulated, as well as real, data examples.
    The proposed truncating lasso method is implemented in the R-package \texttt{grangerTlasso} and is available at \href{http://www.stat.lsa.umich.edu/~shojaie/}{www.stat.lsa.umich.edu/$\sim$shojaie}.
\end{abstract}

\section{Introduction}\label{intro}
  A critical problem in systems biology is to discover causal relationships among components of biological systems. Gene regulatory networks, metabolic networks and cell signalling networks capture causal relationships in cells. Discovery of causal relationships may be only possible through carefully designed experiments, which can be challenging. However, gene regulation is carried out by binding of protein products of transcription factors to \emph{cis}-regulatory elements of genes. Such regulatory mechanisms are evident if the expression levels of gene $X$ is affected by changes in expression levels of gene $Y$. Therefore, time course gene expression data can be used to discover causal relationships among genes and construct the gene regulatory network.

  Different methods have been developed to infer causal relationships from time series data, including dynamic Bayesian Networks \citep{murphy2002dbn} and Granger causality \citep{granger1969icr}. In dynamic Bayesian Networks (DBNs) the state space of Bayesian Networks is expanded by replicating the set of variables in the network by the number of time points. Cyclic networks are then transformed to directed acyclic graphs (DAGs) by breaking down cycles into interactions between variables at two different time points. \citet{ong2002mrp} and \citet{perrin2003gni} among others have applied DBNs to infer causal relationships among components of biological systems.

  On the other hand, the concept of Granger causality states that gene $X$ is Granger-causal for gene $Y$ if the autoregressive model of $Y$ based on past values of both genes is significantly more accurate than the model based on $Y$ alone. This implies that changes in expression levels of genes could be explained by expression levels of their transcription factors. Therefore, statistical methods can be applied to time-course gene expression observations to estimate Granger causality among genes.

  Exploring Granger causality is closely related to analysis of vector autoregressive (VAR) models, which are widely used in econometrics. \citet{yamaguchi2007fmb} and \citet{opgen2007lcn} employed VAR models to learn gene regulatory networks, while \citet{fujita2007mge} proposed a sparse VAR model for better performance in cases when the number of genes, $p$ is large compared to the sample size, $n$. Similar sparse models have also been considered by \citet{mukhopadhyay2007cps}.

  \citet{zou2009gcd} compared the performance of DBNs and Granger causality methods for estimation of causal relationships and concluded that the performance of the two approaches depend on the length of the time series as well as the sample size. The findings of \citet{zou2009gcd} emphasizes the need for sparse models in cases where the sample size is small. In particular, when $p \gg n$, penalized methods often provide better prediction accuracy. \citet{arnold2007tcm} applied the lasso (or $\ell_1$) penalty to discover the structure of graphical models based on the concept of Granger causality and studied the relationship between different key performance indicators in analysis of stock prices.

  Asymptotic and empirical performances of the lasso penalty for discovery of graphical models have been studied by many researchers and a number of extensions of the original penalty have been proposed (we refer to these variants of the lasso penalty as ``lasso-type'' penalties). In particular, to reduce the bias in the lasso estimates, \citet{zou2006ala} proposed the adaptive lasso penalty, and showed that for fixed $p$, if appropriate weights are used, the adaptive lasso penalty can achieve variable selection consistency even if the so-called \emph{irrepresentability} assumption is violated. In fact, it can also be shown that if initial weights are derived from regular lasso estimates, the adaptive lasso penalty is also consistent for variable selection in high dimensional sparse settings \citep{shojaie2009plDAG}.

  The lasso estimate of the graphical Granger model may result in a model in which $X$ is considered to influence $Y$ in a number of different time lags. Such a model is hard to interpret and inclusion of additional covariates in the model may result in poor model selection performance. \citet{lozano2009ggg} have recently proposed to use a group lasso penalty in order to obtain a simpler Granger graphical model. The group lasso penalty takes the average effect of $X$ on $Y$ over different time lags and considers $X$ to be Granger-causal for $Y$ if the average effect is significant. However, this results in significant loss of information, as the time difference between activation of $X$ and its effect on $Y$ is ignored. Moreover, due to the averaging effect, the sign of effects of the variables on each other can not be determined from the group lasso estimate. Hence, whether $X$ is an activator or a suppressor for $Y$ and/or the magnitudes of its effect remain unknown.

  In this paper, we propose a novel \emph{truncating lasso} penalty for estimation of graphical Granger models. The proposed penalty has two main features: (i) it automatically determines the order of the VAR model, i.e. the number of effective time lags and (ii) it performs model simplification by reducing the number of covariates in the model. We propose an efficient iterative algorithm for estimation of model parameters, provide an error-based choice for the tuning parameter and prove the consistency of the resulting estimate, both in terms of sign of the effects, as well as, variable selection properties. The proposed method is applied to simulated and real data examples, and is shown to provide better estimates than alternative penalization methods.

  The remainder of the paper is organized as follows. Section \ref{model}, starts with a discussion of the use of lasso-type penalties for estimation of DAGs as well as a review of the concept of graphical Granger causality. The proposed truncating lasso penalty and asymptotic properties of the estimator are discussed in section \ref{model_3}, while the optimization algorithm is presented in section
  \ref{model_4}. Results of simulation studies are presented in section \ref{sim} and applications of the proposed model to time course gene expression data on E-coli and human cancer cell line (HeLa cells) are illustrated in sections \ref{ecoli} and \ref{hela}, respectively. A summary of findings and directions for future research are discussed in section \ref{conc}.

\section{Model and Methods}\label{model}

\subsection{Graphical Models and Penalized Estimates of DAGs}\label{model_1}
  Consider a graph $\mathcal{G} = (V,E)$, where $V$ corresponds to the set of nodes with $p$ elements and $E \subset V \times V$ is the edge set. The nodes of the graph represent random variables $X_1, \ldots, X_p$ and the edges capture associations amongst them. An edge is called directed if $(i,j) \in E \Rightarrow (j,i) \notin E$ and undirected if $(i,j) \in E \Longleftrightarrow (j,i) \in E$. We represent $E$ through the adjacency matrix $A$ of the graph, a $p \times p$ matrix whose $(j,i)-$th entry indicates whether there is an edge (and its weight) between nodes $j$ and $i$.

  Causal relationships among variables are represented by directed graphs where $E$ consists of only directed edges. Let $\rm{pa}_i$ denote the set of \emph{parents} of node $i$ and for $j \in \rm{pa}_i$, write $j \rightarrow i$. The causal effect of random variables in a DAG can be explained using \emph{structural equation models} \citep{pearl2000caus}, where each variable is modeled as a (nonlinear) function of its parents. The general form of these models is given by:
  \begin{equation}\label{eqnSEM}
    X_i = f_i(\rm{pa}_i,Z_i), \hspace{0.5cm} i=1, \ldots, p
  \end{equation}
  The random variables $Z_i$ are the latent variables representing the unexplained variation in each node. To model the association among nodes of a DAG, we consider a simplification of (\ref{eqnSEM}) where $f_i$ is linear. More specifically, let $\rho_{ij}$ represent the \emph{effect} of gene $j$ on $i$ for $j \in \rm{pa}_i$, then
  \begin{equation}\label{eqnSEMlin}
    X_i = \sum_{j \in \rm{pa}_i}{\rho_{ij} X_j} + Z_{i}, \hspace{0.5cm} i=1, \ldots, p
  \end{equation}
  In the special case, where the random variables on the graph are Gaussian, equations (\ref{eqnSEM}) and (\ref{eqnSEMlin}) are equivalent in the sense that $\rho_{ij}$ are the coefficients of the linear regression model of $X_i$ on $X_j, j \in \rm{pa}_i$. It is known in the normal case that $\rho_{ij}=0, j \notin \rm{pa}_i$.

  For the case of DAGs, it can be shown that when the variables inherit a natural ordering, the likelihood function can be directly written in terms of the adjacency matrix of the DAG. It then follows that the penalized estimate of the adjacency matrix can be found by solving $p-1$ penalized regression problems. To see this, let $\mathcal{X}$ be the $n \times p$ matrix of observations and $S = n^{-1}\T{\mathcal{X}}\mathcal{X}$ be the empirical covariance matrix. Then, the estimate of the adjacency matrix of DAGs under the general weighted lasso (or $\ell_1$) penalty, is found by solving the following $\ell_1$-regularized least squares problems for $i=2, \ldots, p$
     \begin{equation} \label{AOpt_DAG_lasso}
       \hat{A}_{i,1:i-1} = \argmin_{ \theta \in \mathbb{R}^{i-1} } { \left\{
              n^{-1} \| \mathcal{X}_{i} - \mathcal{X}_{1:i-1} \theta \|_2^2 +
              \lambda_i \sum_{j=1}^{i-1}{|\theta_j| w_{ij}} \right\}  }
     \end{equation}
  \noindent where $A_{i,1:i-1}$ denotes the first $i-1$ element of the $i$th row of the adjacency matrix and $w_{ij}$ represents the weights. For the lasso penalty $w_{ij} = 1$ and in case of adaptive lasso $w_{ij} = 1 \vee |\tilde{A}_{ij}|^{-1}$ where $\tilde{A}$ are the initial estimates obtained with the regular lasso penalty.

\subsection{Graphical Granger Causality}\label{model_2}
  Let $X^{1:T} = {\{X\}}_{t=1}^T$ and $Y^{1:T} = {\{Y\}}_{t=1}^T$, be trajectories of two stochastic processes $X$ and $Y$ up to time $T$ and consider the following two regression models:
  \begin{equation}\label{GrangerDef_1}
    Y^{T} = A Y^{1:T-1} + B X^{1:T-1} + \varepsilon^T
  \end{equation}
  \begin{equation}\label{GrangerDef_2}
    Y^{T} = A Y^{1:T-1} + \varepsilon^T
  \end{equation}
  Then $X$ is said to be Granger-causal for $Y$ if and only if the model \ref{GrangerDef_1} results in significant improvements over model \ref{GrangerDef_2}. Graphical Granger models extend the notion of Granger causality among two variables to $p$ variables. In general, let $X_1, \ldots, X_p$ be $p$ stochastic processes and denote by $\mathbf{X}$ the rearrangement of these stochastic processes into a vector time series, i.e. $\mathbf{X}^{t} = \T{( X_1^{t}, \ldots, X_p^{t} )}$. We consider models of the form
  \begin{equation}\label{GrangerDef_4}
    \mathbf{X}^{T} = A^{1} \mathbf{X}^{T-1} + \ldots A^{T-1} \mathbf{X}^1 + \varepsilon^T.
  \end{equation}
  In the graphical Granger model, $X_j^t$ is said to be Granger-causal for $X_i^T$ if the corresponding coefficient, $A^t_{i,j}$ is statistically significant. In that case, there exists an edge $X_j^t \rightarrow X_i^T$ in the graphical model with $T \times p$ nodes.

  Such a model corresponds to a DAG with $T \times p$ variables, in which the ordering of the set of $p$-variate vectors $\mathbf{X}^1, \ldots, \mathbf{X}^T$ is determined by the temporal index and the ordering among the elements of each vector is arbitrary. Lasso-type estimates of DAGs can therefore be used in the context of graphical Granger models in order to estimate the effects of variables on each other. The model in \eqref{GrangerDef_4} is also equivalent to Vector Auto-Regressive (VAR) models \citep[][Chapter 2]{lütkepohl2005new}, which have been used for estimation of graphical Granger causality by a number of researchers, including \citet{arnold2007tcm}.

\subsection{Truncating Lasso for Graphical Granger Models}\label{model_3}
  Consider a graphical model with $p$ variables, observed over $T$ time points, and let $d$ be the order of the VAR model or the effective number of time lags (in \eqref{GrangerDef_4} $d=T-1$). As in section \ref{model_1}, let $\mathcal{X}^t$ denote the design matrix corresponding to $t$-th time point, and $\mathcal{X}^t_i$ be its $i$-th column.

  The truncating lasso estimate of the graphical Granger model is found by solving the following estimation problem for $i = 1, \ldots, p$:
  \begin{equation}\label{problem_NstdLasso}
    \argmin_{\theta^t \in \mathbb{R}^p}{
            n^{-1} \| \mathcal{X}_i^T - \sum_{t=1}^{d}{\mathcal{X}^{T-t}\theta^t} \|_2^2 } +
            \lambda \sum_{t=1}^{d} \Psi^t {\sum_{j=1}^{p}{ |\theta_j^{t}|w_j^t }
            }
  \end{equation}
  \[
    \Psi^1=1, \hspace{0.5cm} \Psi^t = M^{ I\{\|A^{(t-1)}\|_0 < p^2 \beta / (T-t) \} }, \thickspace t \ge 2
  \]
  \noindent where $M$ is a large constant, and $\beta$ is the allowed false negative rate, determined by the user. The choice of $\beta$ and the properties of the resulting estimator are discussed in the remainder of this section.

  To illustrate the main idea behind the truncating lasso penalty, we begin by examining the regular lasso estimate of the graphical Granger model. Using the above notation, the general weighted lasso estimate of the graphical Granger model is found by solving the following $p$ non-overlapping $\ell_1$-regularized least square problems for $i = 1, \ldots, p$:
  \begin{equation}\label{problem_lasso}
    \argmin_{\theta^t \in \mathbb{R}^p}{
            n^{-1} \| \mathcal{X}_i^T - \sum_{t=1}^{d}{\mathcal{X}^{T-t}\theta^t} \|_2^2 } +
            \lambda \sum_{t=T-d}^{T-1}{\sum_{j=1}^{p}{ |\theta_j^{t}|w_j^t }
            }
  \end{equation}

  The weighted lasso penalty suffers from two limitations. Firstly, the order of the VAR model $d$ is often unknown and therefore is set to $T-1$, resulting in $p(T-1)$ covariates in the weighted lasso estimation problem. Moreover, the weighted lasso estimate may potentially include edges from different time points of variable $X_j$ to any given variable $X_i$. To overcome these problems, \citet{lozano2009ggg} proposed to use the group lasso estimate, in which the values of coefficients of each variable over the past time points are grouped. The drawback of group lasso penalty is that information on the time lag between activation of gene $j$ and its effect on gene $i$ is lost. Moreover, the resulting estimate does not provide consistent information about the magnitude and sign of the interaction. Thus, important questions including the activation or inhibition effect of $X_j$ on $X_i$ can not be answered.

  To proposed truncating truncating lasso penalty addresses the above shortcomings of the regular lasso penalty, while preventing the loss of information which occurs if the group lasso penalty is used. The truncating effect of the proposed penalty (imposed by $\Psi^t$) is motivated by the rationale that the number of effects (edges) in the graphical model decreases as the time lag increases. Consequently, if there are fewer than $p^2 \beta/(T-t)$ edges in the $(t-1)$st estimate, all the later estimates are forced to zero. Hence, the truncating lasso penalty provides an estimate of the order of the underlying VAR model. In addition, by applying this penalty, the number of covariates in the model is reduced as the coefficients for effects of genes on each other after the estimated time lag are forced to zero.

  The truncating lasso estimate of the graphical Granger model offers desirable asymptotic properties. In particular, it is shown in the Appendix that the resulting estimate is consistent for variable selection (i.e. the correct edges are estimated with increasing probability, as the sample size increases) in the high dimensional sparse setting. Moreover, with high probability, the signs of the effects are consistently estimated and the order of the underlying VAR model is correctly estimated.

\subsection{Choice of the Tuning Parameter}\label{model_5}
  Estimation of the graphical Granger model using the truncating lasso penalty requires selection of two parameters, $\lambda$ and $\beta$. As mentioned in the previous section, $\beta$ is the allowed rate of false negatives. Therefore, selection of $\beta$ can be based on the cost of false negatives in the specific problem at hand, as well as the sample size; as with any other statistical test, as sample size increases, smaller values of $\beta$ can be considered. A practical strategy for selecting $\beta$ is to first find the lasso (or adaptive lasso) estimate and select $\beta$ so that the desired false negative rate is achieved.

  The second parameter, $\lambda$ is common in all penalized estimation methods. We propose the following error-based choice for selection of $\lambda$. Let $Z^*_q$ be the $(1-q)$-th percentile of the standard normal distribution and consider:
  \begin{equation}\label{tuning}
    \lambda = 2n^{-1/2} Z^*_{\frac{\alpha}{2dp^2}}
  \end{equation}
  \noindent then using the results of \citet{shojaie2009plDAG}, it can be shown that for any value of $n$, this choice of $\lambda$ controls a version of false positive rate at the given level of $\alpha$, provided that columns of the design matrix are scaled so that $n^{-1} \T{\mathcal{X}_i}\mathcal{X}_i = 1$. In section \ref{sim}, we evaluate the performance of the proposed method for a range of values of $\alpha$, and show that the performance is not heavily influenced by that choice.

\subsection{Algorithm and Computational Complexity}\label{model_4}
  In the previous section, we discussed that the truncating lasso estimate of the graphical Granger model in \eqref{problem_NstdLasso} is found by solving $p$ weighted lasso problems. However, the optimization problem in \eqref{problem_NstdLasso} is non-convex and can not be solved directly, especially since the truncating factor $\Psi^{t}$ depends on the values of the coefficients at the previous time points. Here we propose an iterative Block-Relaxation algorithm \citep{DeLee:94}, which can be efficiently used to estimate the parameters of the model.

  The main idea of the algorithm is to further break down each of the $p$ sub-problems into $d$ weighted lasso problems, starting with the observations at the most recent time lag, $T-1$. This iterative process is continued by calculating the truncating factor $\Psi^{t}$ at each $t = 1, \ldots, d$ based on the values of the coefficients at the previous time points and solving a weighted lasso problem over $p$ variables at each time point. Algorithm \ref{graphGrangerIterAlg} outlines the above iterative procedure for finding the estimates of the graphical Granger model.

  \floatstyle{ruled}
  \newfloat{algorithm}{b}{loa}
  \floatname{algorithm}{Algorithm}
  \begin{algorithm}
    \caption[Algorithm]{Iterative Algorithm for Estimation of Truncation Lasso}
    \label{graphGrangerIterAlg}
    \begin{tabbing}
    Repeat for $k = 1, 2, \ldots$ (until convergence) \\
        1. \= For $t = 1, \ldots, d$ \\
            \>1.1. \= Calculate $\Psi^t$ based on estimates in $t' = 1, \ldots, t-1$ \\
            \>1.2. \= Using the most recent estimate $\hat{A}^{t'}$, find: \\
            $
                \hspace{1cm} R^t = \mathcal{X}^T - \sum_{t'=1, t' \ne t}^{d}{\hat{A}^{t'}\mathcal{X}^{T-t'}}
            $ \\
            \>1.3. \= \= For $i = 1, \ldots, p$, let $r := R^{t}_{i}$, and solve
    \end{tabbing}
            $
             \hspace{1cm}   \argmin_{\theta}
                { \left\{
                    n^{-1} \| r - \mathcal{X}^{T-t} \theta^t \|_2^2 + \lambda \Psi^t \sum\nolimits_{j=1}^{p}{ |\theta_j^{t}|w_j^t }
                \right\}  }
            $
  \end{algorithm}

  Unlike the (adaptive) lasso problem, the objective function of the truncating lasso problem is non-convex. Therefore, a global minimum for the resulting optimization problem may not exist. However, the following result shows that the proposed algorithm always converges, although the accumulation point may be a local minimum.

  \begin{lemma}
    Algorithm \ref{graphGrangerIterAlg} converges to a stationary point of the (adaptive) truncating lasso estimation problem.
  \end{lemma}
  \begin{proof}[(Sketch of the Proof)]
    Although the overall objective function is non-convex, each sub-problem is a weighted lasso problem and is therefore convex. On the other hand, the objective function in the (adaptive) truncating lasso problem is separable and it can be shown that assumptions (A1), (B1)--(B3) and (C1) in \citet{tseng2001nondiffBCD} are satisfied. The result follows from Lemma 3.1 and Theorem 5.1 in \citet{tseng2001nondiffBCD}.
  \end{proof}

  Both lasso as well as adaptive lasso problems include $d \times p$ covariates in each penalized regression problem. Therefore, using the \texttt{shooting} Algorithm of \citet{friedman2008rpg} (implemented in the R-package \texttt{glmnet}), estimation of the (adaptive) lasso problem requires $O(n \hat{d}^2 p^2)$ operations, where $\hat{d}$ is the estimate of the order of VAR from the truncating lasso penalty. On the other hand, partitioning over time points reduces the computational burden of each subproblem to $O(n p^2)$. From the general theory of Block-Relaxation algorithms \citep{DeLee:94}, it can be shown that Algorithm \ref{graphGrangerIterAlg} has at least a linear convergence rate. However, in our extensive simulation studies, the algorithm often converges in less than 10 iterations, and for large values of $T$, may require less time than lasso.

  \begin{figure}[t!]
    \centering
    {\includegraphics[scale=0.9, clip=TRUE, trim=2cm 17.7cm 1cm 1.3cm]{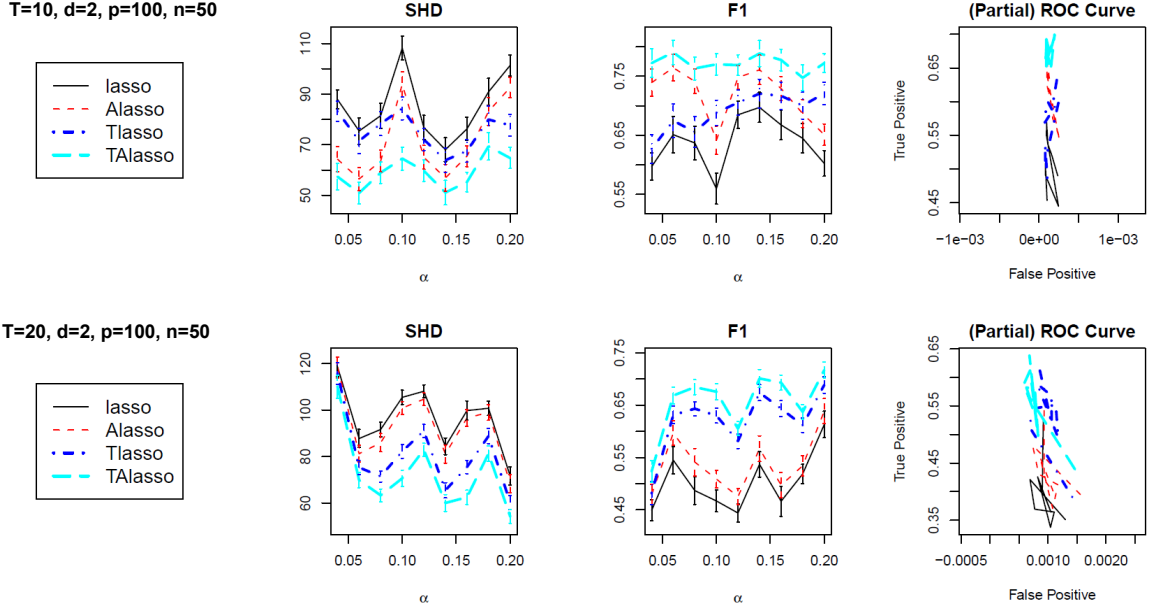}}
    \caption{\footnotesize Mean and standard deviation of performance criteria for \texttt{lasso}, \texttt{Alasso}, \texttt{Tlasso} and \texttt{TAlasso} in estimation of graphical Granger model, with $p=100$, $n=50$ and $d=2$. Top: $T=10$, Bottom: $T=20$.} \label{figsim}
  \end{figure}

\section{Results}\label{results}
\subsection{Simulation Studies}\label{sim}
  We evaluate the performance of the proposed truncating lasso penalty, as well as the lasso and adaptive lasso penalties, in reconstructing the Granger graphical models from time series observations. Several simulations, with different settings of parameters and different network structures are performed, and results of two simulations are presented here. In both simulations $p=100$, and $n=50$ independent and identically distributed (i.i.d.) observations are generated according to a VAR model with order $d=2$, and a Gaussian noise with standard error of $\sigma=0.2$ is added to the observations, i.e. $X^t = \sum_{k=1}^{d}{A^k X^{t-k}}, t=1, \ldots, T$.

  To facilitate the comparison, we control the strength of association among connected nodes (i.e. the non-zero elements of the adjacency matrix) via a single parameter $\rho = 0.7$. In these simulations, the value of the tuning parameter for the penalty coefficient, $\alpha$, is varied from $0.01$ to $0.2$, while the value of the second tuning parameter for the truncating lasso penalty, $\beta$ is fixed at $0.1$. In the first simulation, $T = 10$, while the second simulation includes $T = 20$ time points. Finally, in all simulations (including those not shown), the sparsity level in the network is controlled by setting the total number of edges equal to the sample size $n$.

  To measure the performance of the estimators, we consider three different performance criteria: (1) The Structural Hamming Distance (SHD), (2) the $\text{F}_1$ measure, and (3) The partial ROC plot.

  SHD measures the total number of differences in edges between the estimated and true graphs, with lower values corresponding to better estimates. In other words, $\text{SHD} = \card{(\hat{E} \backslash E)} + \card{(E \backslash \hat{E})}$, , where $\hat{E}$ and $E$ denote the estimated and true edge sets. The $\text{F}_1$ measure is the harmonic mean of \emph{precision} ($P$) and \emph{recall} ($R$) (i.e. $\text{F}_1 = 2PR/(P + R)$) for the estimated graphs, and can be used to compare the performance of estimators in networks with different structures. The value of this summary measure ranges between 0 and 1, with higher values corresponding to better estimates. Finally, the (partial) ROC plot is commonly used to evaluate the performance of classification methods, and in our context illustrates the changes in the true positive rate in comparison to the false positive rate, as the tuning parameter changes.

  The mean and standard deviations of the above criteria, over 50 simulations, for \texttt{lasso}, adaptive lasso (\texttt{Alasso}), truncating lasso (\texttt{Tlasso}) and truncating adaptive lasso (\texttt{TAlasso}) are given in Figure \ref{figsim}. It can be seen that in both cases, the \texttt{TAlasso} provides the best estimate. In addition, as the length of the time series increases, the advantages of the truncating penalty become more pronounced. This improvement is particularly significant in case of small sample sizes, but diminishes in simulations with large $n$, as lasso and adaptive lasso estimates can overcome the curse of dimensionality (data not shown). The above simulation studies provide additional evidence in favor of the adaptive lasso procedure, and indicate that the proposed truncation mechanism offers additional improvement for estimation of Granger causality over the regular version of the lasso penalty. Additional simulations (not shown) with other values of $\rho$ and $\sigma$ indicate that although changes in $\sigma$ do not significantly affect the results, the performance of all methods diminish as $\rho$ decreases. However, the qualitative results presented here are true for other values of $\rho$ and $\sigma$.

  To further investigate the effect of the truncating lasso penalty, it is helpful to examine the adjacency matrix of the estimated graphs. Figure \ref{figsim1_image} provides this information for a small network of size $p=20$. As it can be seen, both lasso and adaptive lasso estimates include additional edges beyond the true order of the VAR model (indicated by small rectangles), while failing to uncover some of the true edges (indicated by small ovals). This is mainly due to the fact that the number of covariates ($d \times p$) is much larger than the sample size $n$. However, by reducing the number of covariates through truncation, the truncating lasso penalty overcomes this shortcoming, and offers improvements in terms of both false positive and false negative rates. 
  
  \begin{figure}[t!]
    \centering
    \scalebox{0.5}
    {\includegraphics[clip=TURE, trim=2cm 9cm 0cm 3cm]{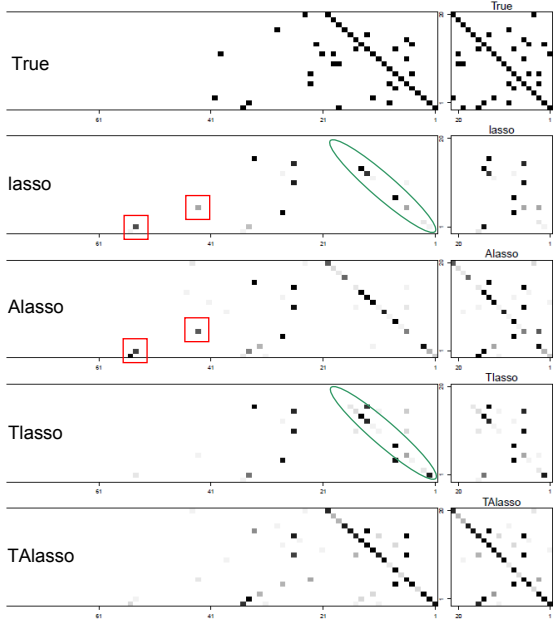}} 
    \caption{\footnotesize Images of the adjacency matrix of the true graph, and estimates from \texttt{lasso}, \texttt{Alasso}, \texttt{Tlasso} and \texttt{TAlasso}. Images on the left correspond to the adjacency matrices of graphical Granger  models (true and estimates) over time, while images on the right represent the cumulative graphical model (the network structure). In the left panel of the true adjacency matrix, a dark pixel in the $(i,j)$th entry at time $t$ represents an edge from $X^{T-t}_j$ to $X^{T}_i$. The gray-scale images for the estimates represent percentage of times where an edge is present in 50 simulations. Significant false positives and negatives are marked with rectangles and ovals, respectively.} \label{figsim1_image}
  \end{figure}

\subsection{Analysis of the Regulatory Network of E-coli}\label{ecoli}
  \citet{kao2004tbd} proposed to use Network Component Analysis to infer the transcriptional regulatory network of Escherichia coli (E-coli). They also provided whole genome expression data over 8 time points with different sample sizes, as well as information about the known regulatory network of E-coli. Figure \ref{figEcoli} represents true and estimated regulatory networks along with performance measures of both \texttt{Alasso}, as well as \texttt{TAlasso} penalties. It can be seen that the rate of recall is improved in the \texttt{TAlasso} estimate, resulting in a higher $\text{F}_1$ measure. The improved performance of the \texttt{TAlasso} penalty in comparison to the \texttt{Alasso} penalty, as well as the overall performance of this estimator, further validate our numerical analysis.

  For comparison, we also provide the estimated regulatory network using our implementation of the group lasso penalty of \citet{lozano2009ggg} (\texttt{grpLasso}). It can be seen that in comparison to \texttt{TAlasso}, \texttt{grpLasso} performs poorly in this example.

  \begin{figure}[t]
    \centering
    \scalebox{0.52}
    {\includegraphics[clip=TURE, trim=1cm 0cm 0cm 0cm]{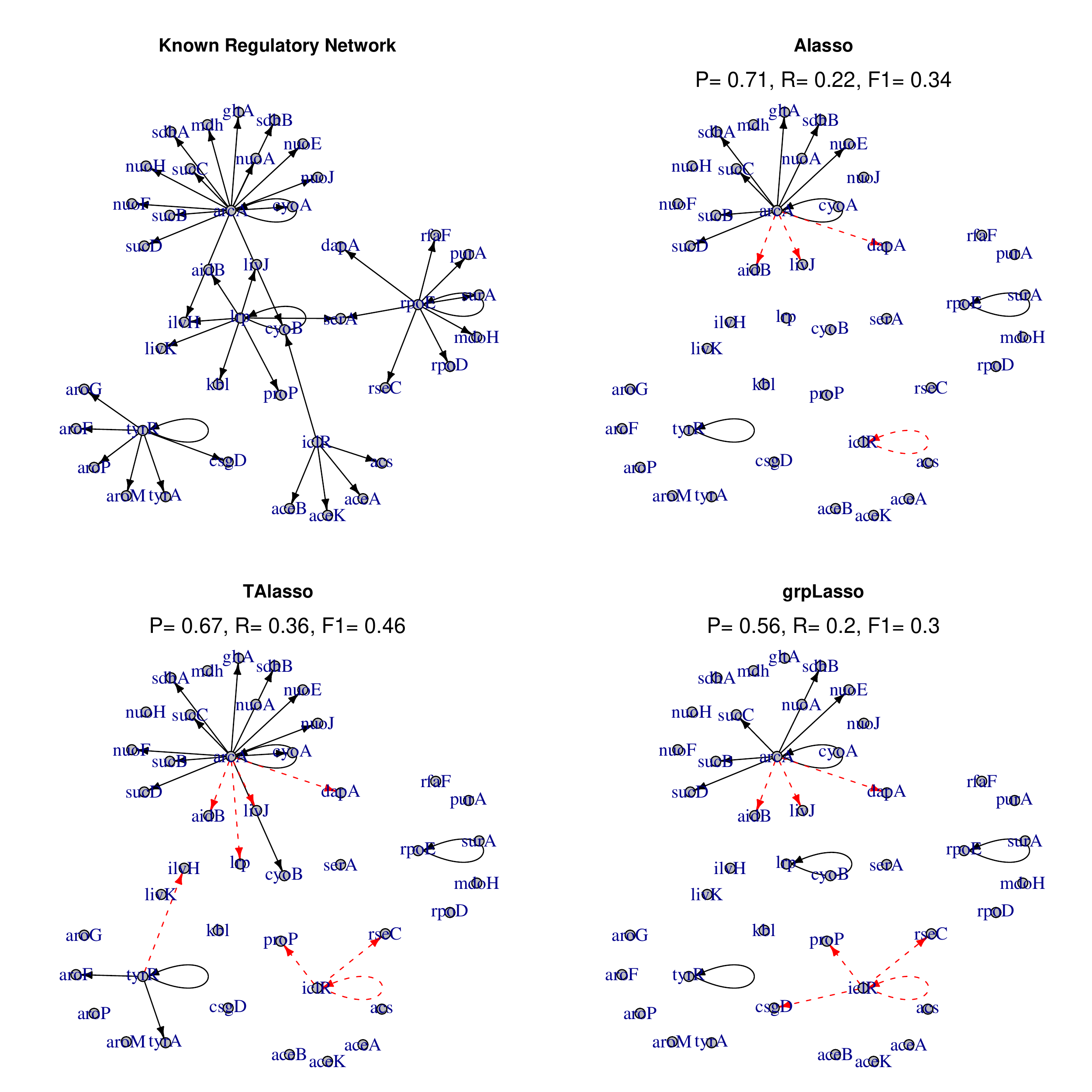}} 
    \caption{\footnotesize Known transcription regulatory network of E-coli along with estimates based on \texttt{Alasso}, \texttt{TAlasso} and \texttt{grpLasso}. True edges (True Positives in estimated networks) are marked with solid black arrows, while False Positives are indicated by dashed red arrows.} \label{figEcoli}
  \end{figure}

\subsection{Analysis of BioGRID Network in HeLa Cells}\label{hela}
  The genome-wide expression of cell cycle genes in human cancer cell lines (HeLa) were analyzed by \citet{whitfield2002hela}. The authors performed different experiments resulting in multiple mRNA time-course samples. \citet{sambo2008cnet} extracted a subset of 9 genes from the human cell cycle genes for which the regulatory network is already determined in the BioGRID database (\href{www.thebiogrid.org}{www.thebiogrid.org}). The authors developed an algorithm for reverse engineering causal gene networks, called \texttt{CNET}, and applied it to this data set. \texttt{CNET} is a search-based algorithm, which searches over the space of possible graphs, in order to find the candidate graph with the highest score.

  This set of 9 genes was also analyzed by \citet{lozano2009ggg}. Figure \ref{figHela} represents the true regulatory network along with estimated networks using our proposed \texttt{TAlasso} estimate, as well as the estimates based on the group lasso and \texttt{CNET} methods. As with the other two groups, we used the third experiment of \citet{whitfield2002hela}, consisting of 47 time points and we considered a maximum time lag of $d=3$. The estimates for group lasso and \texttt{CNET} were reconstructed based on the plots presented by authors, ignoring autoregulatory interactions in the group lasso estimate\footnote{There appears to be a typo in results of \citet{lozano2009ggg}: The BioGRID network should be referred to as the network in Figure 5b (instead of 5a in the paper). Also, the precision, recall and $\text{F}_1$ measures based on the network in Figure 5 are different from the values reported in the paper.}. The best performance is achieved by the \texttt{CNET} algorithm and the authors point out that this result is in line with the best performance obtained in simulated data sets. The performance of the \texttt{TAlasso} method is slightly better than the group lasso estimate. It is important to note that although penalization methods (group lasso and truncating lasso) fail to perform as well as search-based algorithms like the \texttt{CNET} algorithm, they are computationally more efficient and can be used to analyzed large networks, whereas search-based algorithm become intractable for analysis of real-world biological networks.
  \begin{figure}[t]
    \centering
    \scalebox{0.3}
    {\includegraphics{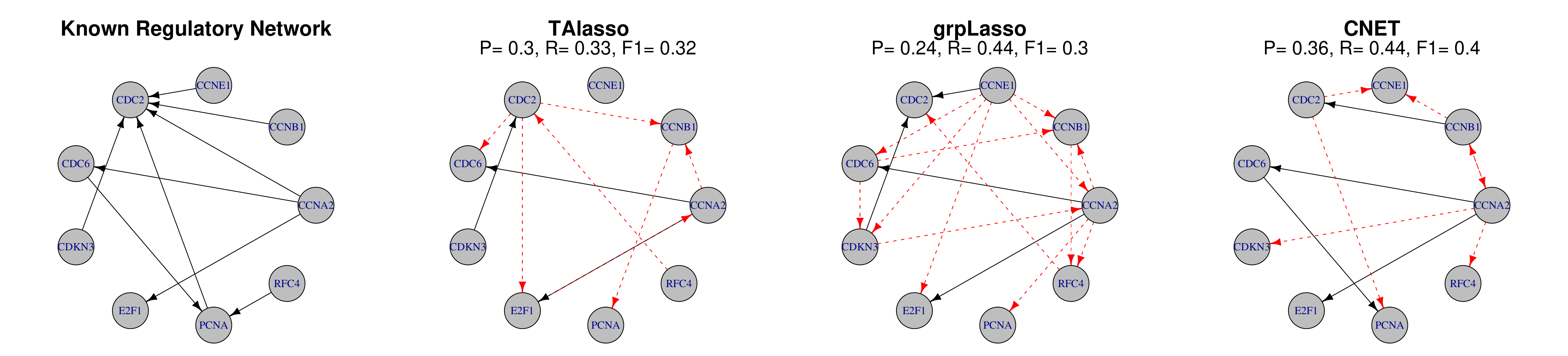}} 
        \caption{\footnotesize Known BioGRID network of human Hela Cell genes along with the estimates based on \texttt{TAlasso}, \texttt{grpLasso} and \texttt{CNET}. True edges (True Positives in estimated networks) are marked with solid black arrows, while False Positives are indicated by dashed red arrows.} \label{figHela}
  \end{figure}

  It can be seen from Figure \ref{figHela} that two of the correctly estimated edges, from CCNA2 to CDC6 and E2F1, are shared in all three estimates and that all true positives of \texttt{TAlasso} are also found by \texttt{grpLasso}. On the other hand, a number of estimated edges not present in the BioGRID network are found in two or more estimates. This may suggest that some of the estimated edges (e.g. the edge from CCNA2 to CCNB1) may represent valid regulatory links that are not included in the BioGRID data set. Validation of such hypotheses requires further investigations.

  A main advantage of the truncating lasso estimate is that it also provides information on the time lag of regulatory effects of transcription factors on other genes. Table \ref{tblHela} provides details of information on effective time lags of effects of genes in the network. Such information provides valuable clues to the underlying regulatory mechanism but is overlooked in the other two methods.

\section{Discussion}\label{conc}
  Estimation of gene regulatory networks is a crucial problem in computational biology. Information conveyed from these networks can be exploited to improve estimation and inference procedures, in particular to determine which pathways are involved in the cell's response to environmental factors or in disease progression \citep[see e.g.][]{shojaie2009NetBasedGSA, shojaie2009NetEnrich}. Such information is also critical in drug development and medicine.

  In this paper, we proposed a novel penalization method, called truncating lasso, for estimation of gene regulatory networks based on the concept of Granger causality. The proposed method can correctly determine the order of the underlying time series, and uses that information to reduce the number of covariates. Such reduction, in turn results in better false positive and false negative rates. Moreover, the proposed method provides information on the time lags of regulatory effects of genes on each other.

  Granger causality is an intuitive concept and its underlying assumption (that expressions of genes at each time point are only affected by expression levels at previous times) can be justified in the study of biological systems. However, from a technical point of view, it may be possible to reformulate the resulting autoregressive model using different causal relationships. A more practical issue concerns the time lags between observations: When observations are observed on coarse time intervals, some of the underlying causal effects may not be distinguishable. The success of reverse engineering algorithms, in particular penalization methods, requires repeated time series observations over fine time grids.

    \begin{table}[h!]
      \caption[HeLa]
      {\footnotesize Time lag of regulatory effects of genes in the estimate of BioGRID network based on the \texttt{TAlasso} algorithm.}\label{tblHela}
      \centering
      {\scriptsize
          \begin{tabular}{|l c|l c|}
             \hline Interaction & Time lag & Interaction & Time lag \\ \hline
             CCNA2 $\rightarrow$ CCNB1 & 1 & CDC2  $\rightarrow$ CDC6  & 1 \\
             CDNK3 $\rightarrow$ CDC2  & 1 & CDC2  $\rightarrow$ E2F1  & 2 \\
             CCNA2 $\rightarrow$ E2F1  & 1 & CCNA2 $\rightarrow$ CDC6  & 2 \\
             CCNB1 $\rightarrow$ PCNA  & 1 & E2F1  $\rightarrow$ CCNA1 & 2 \\
             CDC2  $\rightarrow$ CCNB1 & 1 & RFC4  $\rightarrow$ CDC2  & 2 \\ \hline
          \end{tabular}
    }
  \end{table}

  The method proposed in this paper offers significant improvements over both lasso and adaptive lasso estimates, especially for small to moderate sample sizes. This is achieve by excluding unnecessary covariates from the regression problem. Further improvements may be possible by exploiting the stationarity of the stochastic process in order to take advantage of full information provided in the time series, and should be considered in the future.

\section*{Acknowledgement}
This article will be presented at ECCB10 \href{http://www.eccb10.org/}{http://www.eccb10.org/} and published in Bioinformatics. The authors would like to thank three anonymous referees for constructive comments. The work of George Michailidis was partially supported by NIH grant 1RC1CA145444-0110.

\section*{Appendix}\label{appndx}
  \begin{thm}[Consistency of Truncating Adaptive Lasso]
    Let $s$ be the total number of true edges in the graphical Granger model and suppose that for some $a > 0$, $p = p(n) = O(n^a)$ and $|\rm{pa}_i| = O(n^b)$, where $s n^{2b-1} \log{n}=o(1)$ as $n \rightarrow \infty$. Moreover, suppose that there exists $\nu > 0$ such that for all $n \in \mathbb{N}$ and all $i \in V$, $\Var{\left( X_i^T | X^{T-d:T-1}_{1:p} \right)} \ge \nu$ and there exists $\delta > 0$ and some $\xi > b$ such that for every $i \in V$ and for every $j \in \rm{pa}_i$, $| \pi_{ij} | \ge \delta n^{-(1-\xi)/2}$, where $\pi_{ij}$ is the partial correlation between $X_i$ and $X_j$ after removing the effect of the remaining variables.

    Assume that $\lambda \asymp d n^{-(1-\zeta)/2}$ for some $b < \zeta < \xi$ and $d > 0$, and the initial weights are found using lasso estimates with a penalty parameter $\lambda^0$ that satisfies $\lambda^0 = O(\sqrt{\log p /n})$. Also, for some large positive number $g$, let $\Psi^t = g \exp{(n I\{\|A^{(t-1)}\|_0 < p^2 \beta / (T-t) \} )}$ (i.e. $M=g e^n$). Then if true causal effects diminish over time,
        \begin{enumerate}
        \item[(i)]
            With probability converging to 1, no additional Granger-causal effects are included in the model and the signs of such effects are correctly estimated.
        \item[(ii)]
            With probability asymptotically larger than $1-\beta$, true Granger-causal effects and the order of the VAR model are correctly determined.
    \end{enumerate}
  \end{thm}
  \begin{proof}
    If $\beta=0$, inclusion of the true causal effect, exclusion of incorrect effects and consistency of signs of effects follow from Theorem 3 of \citet{shojaie2009plDAG}. Since $\beta$ has no effect on the probability of false positive, this proves (i).

    For any given $\beta>0$, suppose $t_0$ is the smallest $t$ for which $\|A^{(t-1)}\|_0 < p^2 \beta / (T-t)$. Then for $t < t_0$ $\Psi^t = 1$ and has no effect on the estimate. Let $t \ge t_0$. Then using the KKT conditions, a coefficient is included in the weighted lasso estimate only if $|2n^{-1} \T{(\mathcal{X}^t_j)} (\mathcal{X}^T_i - \mathcal{X}^t \theta^t)| > \Psi^t \lambda w^t_j$. However, $\T{(\mathcal{X}^t_j)} (\mathcal{X}^T_i - \mathcal{X}^t \theta^t)$ is stochastically smaller than
    $\T{(\mathcal{X}^t_j)} \mathcal{X}^T_i$, which is in turn a polynomial function of $n$. On the other hand, $\lambda$ and $w^t_j$ are also polynomial functions of $n$, whereas $\Psi^t$ increases exponentially as $n \rightarrow \infty$. Hence, for all $j=1, \ldots, p$ and $t \ge t_0$, there exists an $n$ such that $|2n^{-1} \T{(\mathcal{X}^t_j)} (\mathcal{X}^T_i - \mathcal{X}^t \theta^t)| < \Psi^t \lambda w^t_j$ and therefore, $A^t = 0, t \ge t_0$. However, since the number of true causal effects diminish over time, the total number of true edges in time lags $t \ge t_0$ is less than $\beta$. This proves the first part of (ii).

    Finally, to prove that the order of VAR is correctly estimated, i.e. $d = t_0-1$, we consider two complementary events: $d < t_0-1$ and $d > t_0-1$. Prior to $t_0$, false positives occur with exponentially small probability, hence, the probability that $d < t_0-1$, is negligible. On the other hand, $d > t_0-1$ only if true edges are not included in $\hat{A}^t_0$ and as a result $\|\hat{A}^{(t_0-1)}\|_0 < p^2 \beta / (T-t_0)$. But false negatives occur if true edges vanish in the adaptive lasso estimate. However, adaptive lasso finds the true edges with exponentially large probability, hence, $\mathbb{P}(d < t_0-1) \ge 1 - \beta - O(\exp(-c n^d))$ for constants $c \text{ and } d$. This completes the proof.
  \end{proof}

\bibliographystyle{natbib}
\bibliography{ShojaieBib}
\end{document}